%% file: collas2025_conference.tex
\documentclass{article} %
\usepackage{collas2025_conference,times}
\usepackage{easyReview}

\input{math_commands.tex}

\usepackage{hyperref}
\hypersetup{
    colorlinks=true,
    linkcolor=red,
    filecolor=magenta,
    urlcolor=blue,
    citecolor=purple,
    pdftitle={Overleaf Example},
    pdfpagemode=FullScreen,
    }

\usepackage{microtype}
\usepackage{textgreek}
\usepackage{amsfonts,amsmath,amssymb}   %
\usepackage{graphicx,subfigure}         %
\usepackage{booktabs,tabularx}          %
\usepackage{nicefrac}                   %
\usepackage{multirow}
\usepackage{lipsum}  
\usepackage[font=footnotesize]{caption}
\usepackage{algorithm}
\usepackage{algorithmic}
\usepackage{adjustbox}

\usepackage{amsthm}
\usepackage{xspace}

\usepackage{enumitem}
\usepackage{wrapfig}
\usepackage{subcaption}

\newtheorem{theorem}{Theorem}[section]

\newcommand{\ours}{ROAM\xspace}

\usepackage{titlesec}
\titlespacing\subsection{0pt}{4pt plus 4pt minus 2pt}{2pt plus 2pt minus 2pt}

\newcommand{\D}{\mathcal{D}}

\renewcommand{\E}{\mathbb{E}}
\newcommand{\mc}{\mathcal}
\definecolor{mygreen}{rgb}{0.1, 0.6, 0.1}

\title{Adapt On-the-Go: Behavior Modulation for Single-Life Robot Deployment}

\author{%
  Annie S. Chen$^*$ \\
  Stanford University \\
  \And
  Govind Chada$^*$ \\
  Stanford University \\
  \And
  Laura Smith \\
  UC Berkeley \\
  \And
  Archit Sharma \\
  Stanford University \\
  \AND
  Zipeng Fu \\
  Stanford University \\
  \And
  Sergey Levine \\
  UC Berkeley \\
  \And
  Chelsea Finn \\
  Stanford University
}

\collasfinalcopy %

\begin{document}

\maketitle

\begin{abstract}
To succeed in the real world, robots must cope with situations that differ from those seen during training. 
We study the problem of adapting on-the-fly to such novel scenarios during deployment, by drawing upon a diverse repertoire of previously-learned behaviors. 
Our approach, RObust Autonomous Modulation (\ours), introduces a mechanism based on the perceived value of pre-trained behaviors to select and adapt pre-trained behaviors to the situation at hand. 
Crucially, this adaptation process all happens within a single episode at test time, without any human supervision. 
We provide theoretical analysis of our selection mechanism and demonstrate that \ours enables a robot to adapt rapidly to changes in dynamics both in simulation and on a real Go1 quadruped, even successfully moving forward with roller skates on its feet.
Our approach adapts over 2x 
as efficiently compared to existing methods when facing a variety of out-of-distribution situations during deployment by effectively choosing and adapting relevant behaviors on-the-fly. 

\end{abstract}

\begin{figure*}[h!]
    \centering
    \includegraphics[width=0.8\textwidth]{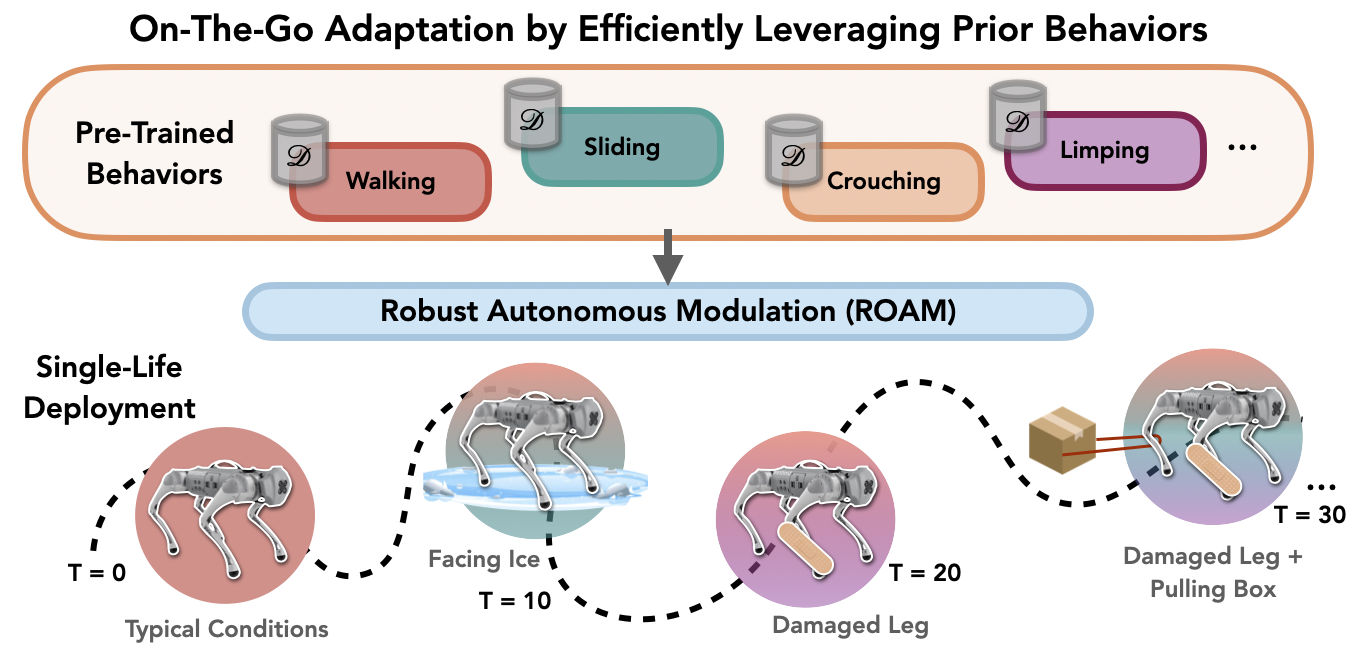}
    \caption{
        \small 
        \textbf{On-The-Go Adaptation via Robust Autonomous Modulation (ROAM).} 
        An agent will inevitably encounter a wide variety of situations during deployment, and handling such situations may require a variety of different behaviors. We propose Robust Autonomous Modulation (\ours), which dynamically employs  relevant behaviors as different situations arise during deployment. 
    }
    \vspace{-4mm}
    \label{fig:teaser}
\end{figure*}

\section{Introduction}
\label{sec:intro}

A major obstacle to the broad application of sequential decision-making agents is their inability to adapt to unexpected circumstances, which limits their uses largely to tightly controlled environments. 
Even equipped with prior experience and pre-training, agents will inevitably encounter out-of-distribution (OOD) situations at deployment time that may require a large amount of on-the-fly adaptation.
In this work, we aim to enable a robot to autonomously handle novel scenarios encountered during deployment, while drawing upon a diverse set of pre-trained behaviors that may improve its versatility. We hypothesize that by doing some of the adaptation in the space of pre-trained behaviors (rather than only in the space of parameters), we can much more quickly react and adapt to novel circumstances.
Accumulating a set of different behavior policies is a relatively straightforward task through online or offline episodic reinforcement learning, using different reward functions or skill discovery methods.
However, the on-the-fly selection and adaptation of these behaviors during deployment, particularly in novel environments, present significant challenges.
Consider tasking a quadrupedal robot that has acquired many different behaviors, e.g., walking, crouching, and limping via training in simulation, with a search-and-rescue mission in the real world. When deployed on this task with unstructured obstacles, the robot may bump into an obstacle and damage its leg, and it must be able to dynamically adapt its choice of behaviors to continue its mission with the damage.

Existing adaptation methods often operate within the standard, episodic training paradigm where the agent is assumed to be reset and have another chance to attempt the task each time~\citep{cully2015robots,song2020rapidly,julian2020never}.
However, these are idealized conditions that in practice often rely on human intervention. In the above search-and-rescue example, the robot's state cannot be arbitrarily restored; during deployment, the robot cannot be repaired, and it is not feasible for a human to fetch it if it gets stuck in a situation it is not equipped to handle. This situation necessitates adapting at test time both quickly and autonomously, to succeed at the task within a single episode. Therefore, we frame our problem setting as an instantiation of single-life deployment~\citep{chen2022you}, where the agent is given prior knowledge but evaluated on its ability to successfully complete a given task later in a `single-life' trial, during which there are no human-provided resets. The robot is provided with a diverse set of prior behaviors trained through episodic RL, 
and the single life poses a sequence of new situations.

To solve this problem, the robot must identify during deployment which behaviors are most suited to its situation at a given timestep and have the ability to fine-tune those behaviors in real time, as the pre-trained behaviors may not optimally accommodate new challenges. 
We introduce a simple method called RObust Autonomous Modulation (\ours), which foremost aims to quickly identify the most appropriate behavior from its pre-trained set at each point in time during single-life deployment. Rather than introducing an additional component like a high-level controller to select behaviors, we leverage the value function of each behavior. The value function already estimates whether each policy will be successful, but may not be accurate for states that were not encountered during training of that behavior. Therefore, prior to deployment, we fine-tune each behavior's value function with a regularized objective that encourages behavior identifiability: the regularizer is a behavior classification loss. Then, at each step during deployment, \ours samples a behavior proportional to its classification probability, executes an action from that behavior, and optionally fine-tunes the selected behavior for additional adaptation.

The main contribution of this paper is a simple algorithm for autonomous, deployment-time adaptation to novel scenarios. 
At a given state, with the additional cross-entropy regularizer, \ours can constrain each behavior’s value to
be lower than the value of behaviors for which that state appears more frequently. Consequently, our method incentivizes each behavior to differentiate between familiar and unfamiliar states, allowing \ours to better recognize when a behavior will be useful. 
We conduct experiments on both simulated locomotion tasks and on a real Go1 quadruped robot. 
In simulation, our method completes the deployment task more than two times faster on average than existing methods, including two prior methods designed for fast adaptation. We also empirically analyze how the additional cross-entropy term in the loss function of \ours contributes to more successful utilization of the prior behaviors.
Furthermore, \ours enables a Go1 robot to adapt on-the-go to various OOD situations without human interventions or supervision in the real world. With \ours, the robot can successfully pull heavy luggage, pull loads with dynamic weights, and even slide forward with two roller skates on its front feet, even though it never encountered loads or wore roller skates during training.

\section{Related Work}
\label{sec:related}
We consider the problem of enabling an agent to act robustly when transferring to unstructured test-time conditions that are unknown at train-time.
One common instantiation of this problem is transfer to different dynamics, e.g., in order to transfer policies trained in simulation to the real world. 
A popular approach in achieving transfer under dynamics shift is domain randomization, i.e., randomizing the dynamics during training~\citep{cutler2014reinforcement,rajeswaran2016epopt,sadeghi2016cad2rl,tobin2017domain,peng2018sim,tan2018sim,yu2019sim,akkaya2019solving,xie2021dynamics,margolis2022rapid,haarnoja2023learning} to learn a robust policy. 
Our approach is similar in that it takes advantage of different MDPs during training; however, a key component of \ours is to leverage and modulate diverse skills rather than a single, robust policy. 
We find in Section~\ref{sec:experiments} that challenging test-time scenarios may require distinct behaviors at different times, and we design our method to be robust to those heterogeneous conditions.

Another class of methods involve training policies that exhibit different behavior when conditioned on dynamics parameters, then distilling these policies into one that can be deployed in target domains where this information is not directly observable. The train-time supervision can come in the form of the parameter values~\citep{yu2017preparing,ji2022concurrent} or a learned representation of them~\citep{lee2020learning,kumar2021rma}. Thereafter, there are several ways prior work have explored utilizing online data to identify which behavior is appropriate on-the-fly, e.g., using search in latent space~\citep{yu2019sim,peng2020learning,yu2020learning}, or direct inference using proprioceptive history~\citep{lee2020learning,kumar2021rma,fu2023deep}, or prediction based on egocentric depth~\citep{miki2022learning,agarwal2023legged,zhuang2023robot,yang2023neural}.
In this work, we do not rely on domain-specific information nor external supervision for when particular pre-trained behaviors are useful. 
Moreover, in contrast to many of the above works, we focus on solving tasks that may be OOD for all prior behaviors individually. 

Meta-RL is another line of work that achieves rapid adaptation without privileged information by optimizing the adaptation procedure during training~\citep{wang2016learning,duan2016rl,finn2017model,nagabandi2018learning,houthooft2018evolved,rothfuss2018promp,rusu2018meta,mendonca2020meta,song2020rapidly} to be able to adapt quickly to a new situation at test-time. 
Meta-RL and the aforementioned domain randomization-based approaches entangle the training processes with the data collection, requiring a lot of \textit{online} samples that are collected in a particular way for pre-training.
The key conceptual difference in our approach is that \ours is \textit{agnostic to how the pre-trained policies and value functions are obtained}. 
Moreover, while meta-RL methods often use hundreds of pre-training tasks, or more, our approach can provide improvements in new situations even with a relatively small set of pre-trained behaviors (e.g. just 4 different behaviors improve performance in Section~\ref{sec:experiments}).
Other transfer learning approaches adapt the weights of the policy to a new environment or task, either through rapid zero-shot adaptation~\citep{hansen2020self, yoneda2021invariance,chen2022you} or through extended episodic online training~\citep{khetarpal2020towards,rusu2016progressive,eysenbach2020off,xie2020deep, xie2021lifelong}. 
Unlike these works, we focus on adaptation within a single episode to a variety of different situations.

Another rich body of work considers how to combine prior behaviors to solve long-horizon tasks, and some of these works also focus on learning or discovering useful skills \citep{gregor2016variational,achiam2018variational,eysenbach2018diversity,nachum2018near,sharma2019dynamics,baumli2021relative,laskin2022cic,park2023predictable}. Many of these methods involve training a high-level policy that learns to compose learned skills into long-horizon behaviors \citep{bacon2017option,peng2019mcp,lee2019learning,sharma2020learning,strudel2020learning,nachum2018data,chitnis2020efficient,pertsch2021guided,dalal2021accelerating,nasiriany2022augmenting}. We show in Section~\ref{sec:experiments} that such a high-level policy may not be effective for on-the-go behavior selection. 
Moreover, our work does not focus on where the behaviors come from -- they could be produced by these prior methods, or their rewards could be specified manually. Instead, we focus on quickly selecting and adapting the most suitable skill in an OOD scenario, without requiring an additional online training phase to learn a hierarchical controller.

\section{Preliminaries}
\label{sec:prelim}
In this section, we describe some preliminaries and formalize our problem statement. 
We are given a set of $n$ prior behaviors,
where each behavior $b_i$ is trained through episodic RL for a particular MDP $\mc{M}_i = (\mc{S}, \mc{A}, \mc{P}_i, \mc{R}_i, \rho_0, \gamma)$ where $\mc{S}$ is the state space, $\mc{A}$ is the agent's action space, $\mc{P}_i(s_{t+1} | s_t, a_t)$ represents the environment's transition dynamics, $\mc{R}_i: \mc{S} \rightarrow \mathbb{R}$ indicates the reward function, $\rho_0: \mc{S} \rightarrow \mathbb{R}$ denotes the initial state distribution, and $\gamma \in [0, 1)$ denotes the discount factor. Each of the $n$ MDPs $\mc{M}_i$ has potentially different dynamics and reward functions $\mc{P}_i$ and $\mc{R}_i$, often leading to different state visitation distributions. 
Each behavior corresponds to a policy $\pi_i$ and a value function $V_i$ as well as a buffer of trajectories $\tau \in \mc{D}_i$ collected during this prior training and relabeled with the reward $\mc{R}_{\text{target}}$ from the target MDP.
At test time, the agent interacts with a target MDP defined by $\mc{M}_{\text{target}} = (\mc{S}, \mc{A}, \mathcal{P}_{\text{target}}, \mc{R}_{\text{target}}, \rho_0, \gamma)$, which presents an aspect of novelty not present in any of the prior MDPs, in the form of new dynamics $\mathcal{P}_{\text{target}}(s_{t+1} \mid s_t, a_t)$, which may \textit{change over the course of the test-time trial}. 
We operate in a single-life deployment setting~\citep{chen2022you} that aims to maximize $J = \sum_{t=0}^h \gamma^t \mc{R}(s_t)$, 
where $h$ is the trial horizon, which may be $\infty$. 
The agent needs to complete the desired task in this target MDP in a single life without any additional supervision or human intervention by effectively selecting and adapting the prior behaviors to the situation at hand.

Off-policy reinforcement learning (RL) algorithms train a parametric Q-function, represented as $Q_\theta(s, a)$, via iterative applications of the Bellman optimality operator, expressed as 
\[ B^* Q(s, a) = r(s, a) + \gamma \mathbb{E}_{s' \sim P(s'|s, a)} \left[ \max_{a'} Q(s', a') \right]. \]
In actor-critic frameworks, a separate policy is trained to maximize Q-values. These algorithms alternate between policy evaluation, which involves the Bellman operator $ B^\pi Q = r + \gamma P^\pi Q$, and policy improvement, where the policy $\pi(a|s)$ is updated to maximize expected Q-values. We use a state-of-the-art off-policy actor-critic algorithm RLPD \citep{ball2023rlpd} as our base algorithm for pre-training and fine-tuning, which builds on regularized soft actor-critic \citep{haarnoja2018soft}.

\begin{figure*}[t!]
    \centering
    \includegraphics[width=1.0\textwidth]{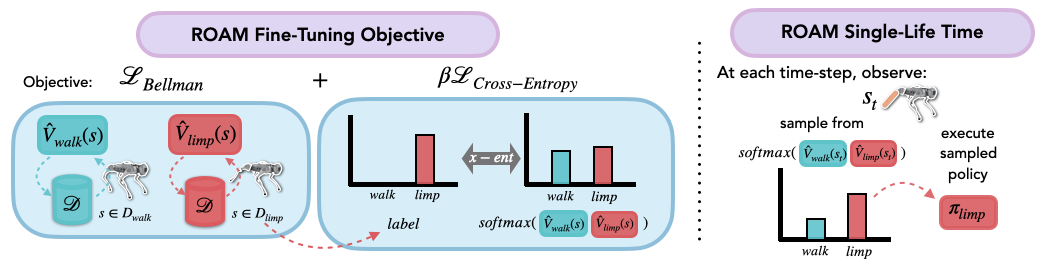}
    \caption{\small \textbf{Robust Autonomous Modulation (ROAM).} During the initial fine-tuning phase of \ours, we fine-tune each behavior using its existing data buffer and standard Bellman error, with an additional cross-entropy loss between the softmax values of all behaviors and the behavior index (as the label) from which the state was visited. Then at test-time, at each time-step, we sample from the softmax distribution of the behaviors' values at the current state and execute the sampled policy.}
    \label{fig:method}
    \vspace{-5mm}
\end{figure*}

\section{Robust Autonomous Modulation}
\label{sec:method}

We now present our method, Robust Autonomous Modulation (\ours), which fine-tunes value functions with an additional loss and provides a mechanism for choosing among them, so that at deployment time, the agent can quickly react to its current situation \textit{at every timestep} by honing in on the most appropriate behavior from its prior behaviors.
Our key observation is that with proper regularization, value functions provide a good indication of how well different behaviors will perform, so we can leverage them to quickly identify appropriate behaviors, which circumvents the need to learn a separate meta-controller or adaptation module.

\paragraph{Behavior Modulation using Value Functions.} 

The core idea of our method is to directly utilize the expressive power of value functions for adaptive behavior selection, as they inherently contain detailed information about potential rewards associated with different states for each behavior. 
We propose to use value functions of the behaviors to select the most appropriate behavior for a given state---namely, at each timestep during deployment, choosing one of the behaviors with the high values at that state. 
Since we already have access to value functions from pre-training the behaviors, this approach does not require any additional training or data collection.
Using value functions as the proxy selector also gives much more versatility to the selection mechanism, which can be flexibly controlled on the go by updating the value functions of different behaviors. 
However, naively using the pre-trained value functions may not lead to high-reward behaviors, due to overestimation of the value functions on new states, as the pre-trained Q-functions may not generalize well to OOD situations.
Recent studies in offline RL, for example, have observed an overestimation bias~\citep{levine2020offline} due to OOD training-time actions, which can lead to poor performance when deploying the policy in a new environment.
To mitigate these issues, works in offline RL have proposed a number of various modifications aimed at regularizing either the policy or the value function \citep{kumar2020conservative,yu2020mopo,yu2021combo}.
Although our setting is different, as we deal with OOD states, we face a similar problem of poor generalization of the value functions of the prior behaviors to unfamiliar situations.
In the following section, we describe how we can conservatively regularize the value functions to improve their generalization.

\paragraph{Fine-Tuning Value Functions with \ours.}

We desire value functions that accurately reflect the expected reward of using a behavior at a given state. 
Thus, we would like to fine-tune the value functions of the behaviors to minimize overestimation of the values at states for which a different behavior is more suitable. 
We can do so by incentivizing the value functions of the behaviors to be higher for states that are visited by that behavior and lower for states visited by other behaviors.

Given a set of prior behaviors $\mathcal{B}$ with policies $\pi_i$ and critics $Q_i$, we first fine-tune the value functions of the behaviors with an additional cross-entropy loss on top of the Bellman error that takes in the values at a given state as the logits. Values at a given state $s$ for behavior $b_i$ are obtained by averaging $Q_i(s, a)$ over $N=5$ sampled actions $a \sim \pi_i(\cdot \mid s)$.
More formally, with each prior data buffer $\mc{D}_i$, we fine-tune the critic $Q_i(s, a)$ of each behavior $b_i$ with the following update:
\begin{equation}
    \begin{split}
        \mc{L}_{\text{fine-tune}} &= (1 - \beta) \mc{L}_{\text{Bellman}} + \beta \mc{L}_{\text{cross-entropy}} \\
        &= (1 - \beta) \sum_i \E_{s, a, s'\sim \mc{D}_i} \left[ \left( r(s, a) \right. \right.
         + \left. \left. \gamma \E_{a' \sim \pi_i(a' | s')} Q_i(s', a') - Q_i(s, a) \right)^2 \right] \\
        &\quad + \beta \sum_{j} \E_{s \sim \D_j} \left[ -\log \frac{\exp(V_j(s))}{\sum_{k=1}^n \exp(V_k(s))} \right],
    \end{split}
    \label{eq:pretrain}
\end{equation}
where $0 < \beta < 1$ is a hyperparameter, $V_i(s) = \E_{a \sim \pi_i(a|s)}[Q_i(s, a)]$ is the average value of behavior $b_i$ at state $s$, and 
$\D_i$ is a replay buffer collected by behavior $b_i$.
Consider the derivative of the cross-entropy term with respect to the value functions $V_i(s)$ and $V_k(s)$, where $i \neq k$, for a state $s$ in $\D_i$:
\[ \frac{\partial{\mc{L}_{\text{cross-entropy}}}}{\partial{{V_i(s)}}} = \frac{\exp(V_i(s))}{\sum_{j=1}^n \exp(V_j(s))} - 1 < 0,
\frac{\partial{\mc{L}_{\text{cross-entropy}}}}{\partial{V_k(s)}} = \frac{\exp(V_k(s))}{\sum_{j=1}^n \exp(V_j(s))} > 0.\]
So when minimizing the cross-entropy loss, the value function $V_i(s)$ will be pushed up (since its derivative is negative), and $V_k(s)$ for $k \neq i$ will be pushed down.
Thus, the cross-entropy loss term in Equation~\ref{eq:pretrain} pushes up the value functions of the behaviors for states that are visited by that behavior and pushes down for states that are visited by other behaviors.
The value functions are then less likely to overestimate at OOD states, enabling the behaviors to specialize in different parts of the state space, which will help us at test time to better infer an appropriate behavior from the current state.

\paragraph{Full Procedure and Single-Life Deployment.}

To summarize the full procedure of \ours, we are given a set of policies $\pi_i$ and critics $Q_i$, and a set of prior data buffers $\mc{D}_i$ for each behavior, which have been relabeled with the target MDP reward.
Alternatively, this can be relaxed and we can assume that we are given a set of prior data buffers $\mc{D}_i$ for each behavior, and we can train the policies $\pi_i$ and critics $Q_i$ using these buffers with offline RL.

We then fine-tune the value functions of the behaviors, \( V_{i, \text{orig}} \), by incorporating the additional cross-entropy loss term from Equation~\ref{eq:pretrain}, yielding updated value functions $V_i$ for behavior selection. We retain copies of the original value functions to ensure that, if we fine-tune the policies at test time, they remain unaffected by the newly learned value functions. This prevents the ROAM cross-entropy loss from influencing policy optimization, ensuring that the policy updates rely only on their original training objectives.

At each timestep during test time, a behavior is sampled from a softmax distribution over the behaviors' values \( V_i \) at the current state. Specifically, given the current state \( s \), an action \( a_i \sim \pi_i(a|s) \) is drawn from behavior \( b_i \), with selection probability proportional to \( \exp(V_i(s)) \). The transition \( (s_t, a_t, r_t, s_{t+1}) \) is then stored in the online buffer \( \mathcal{D}^i_{\text{online}} \) for behavior \( b_i \), and both the critic \( V_i \), policy \( \pi_i \), and the original value function \( V_{i, \text{orig}} \) are fine-tuned using data from \( \mathcal{D}^i_{\text{online}} \). Importantly, the fine-tuned value functions influence only the selection of behaviors and do not modify the policies themselves, ensuring that policy optimization remains unbiased during test-time adaptation.

The ROAM fine-tuning objective and single-life deployment are depicted in Figure~\ref{fig:method} and the full algorithm is summarized in Algorithms~\ref{algoblock1} and \ref{algoblock2} in Appendix~\ref{sec:app-method}.
In the next subsection, we additionally provide theoretical analysis of \ours to show that the additional cross-entropy loss in \ours will lead to more suitable behaviors being chosen.

\subsection{Theoretical Analysis}
\label{sec:analysis}

Next, we theoretically analyze \ours to show that the additional cross-entropy loss in \ours will lead to more suitable behaviors being chosen at each timestep. 
In particular, \ours rescales the value functions of the behaviors so that they are less likely to overestimate in states that are out of distribution for that behavior.
Our main result, given in Theorem~\ref{thm:main},
is that with \ours, for some weight $\beta > 0$ on the cross-entropy term, at a given state, \ours constrains each behavior's value to be lower than the value of behaviors for which that state appears more frequently. 
This theorem gives us \textit{conservative generalization} by reducing value overestimation in unfamiliar states--specifically, our chosen behavior will not have worse performance than the most familiar behavior. 

This theoretical result holds for all local extrema. Under perfect optimization, ROAM will reach a local minimum at convergence and the result would hold in this case. Our analysis implicitly assumes a simplified tabular setting in the absence of approximation and sampling errors, as done in prior theoretical analysis of deep RL algorithms, e.g. some of the analysis in \citep{kumar2020conservative}. We take the full gradient of the TD error in the proof although in practice, only the semi-gradient is used during optimization. Regardless, the optimal value function is still a minimum of $\mc{L}_{\text{fine-tune}}$ and should satisfy the Bellman equation in the loss function at convergence. We also assume that for all states $s$, we have $V_{\text{freq}}(s) - V_i(s) =\left( R_{\text{freq}}(s) + \gamma E_{s'} V_{\text{freq}}(s')\right) - \left( R_i(s) + \gamma E_{s'} V_i(s')\right) < \eta < \infty$. 
A full proof of the statement in this section is presented in Appendix~\ref{sec:app-theory}.

\begin{theorem}
    Let $p_i(s)$ denote the state visitation probability for a behavior $b_i$ at state $s$. For any state $s$ that is out of distribution for behavior $b_i$ and is in distribution for another behavior $b_j$, i.e. $p_i(s) \ll p_j(s)$, if $1 > \beta > 0$ is chosen to be large enough, then the value of behavior $b_i$ learned by \ours will be bounded above compared to value of behavior $b_j$, i.e.  $V_i(s) \leq V_j(s)$.
    \label{thm:main}
\end{theorem}

As a result, for any states $s$ that are out of distribution for behavior $b_i$ and in distribution for a different behavior $b_j$, if we choose $\beta$ large enough, the value learned will not overestimate the value compared to behavior $b_j$.
Thus, at each time step, if one or more behaviors are familiar with the current state, the performance of the chosen behavior will not be much worse than its value function estimate. In this manner, \ours adjusts value estimates based on degree of familiarity, mitigating overestimation risks.
 The ability to adjust the $\beta$ parameter offers a flexible framework to optimize for the behavior with the highest value at a given state, which will be at least as suitable as the most familiar behavior.

In the next section, we find empirically that after fine-tuning with the additional cross-entropy loss, \ours is able to effectively select a relevant behavior for a given state on-the-fly, leading to robust and fast adaptation to OOD situations.

\section{Experimental Results}
\label{sec:experiments}

\begin{figure*}
    \centering
    \includegraphics[width=0.9\textwidth]{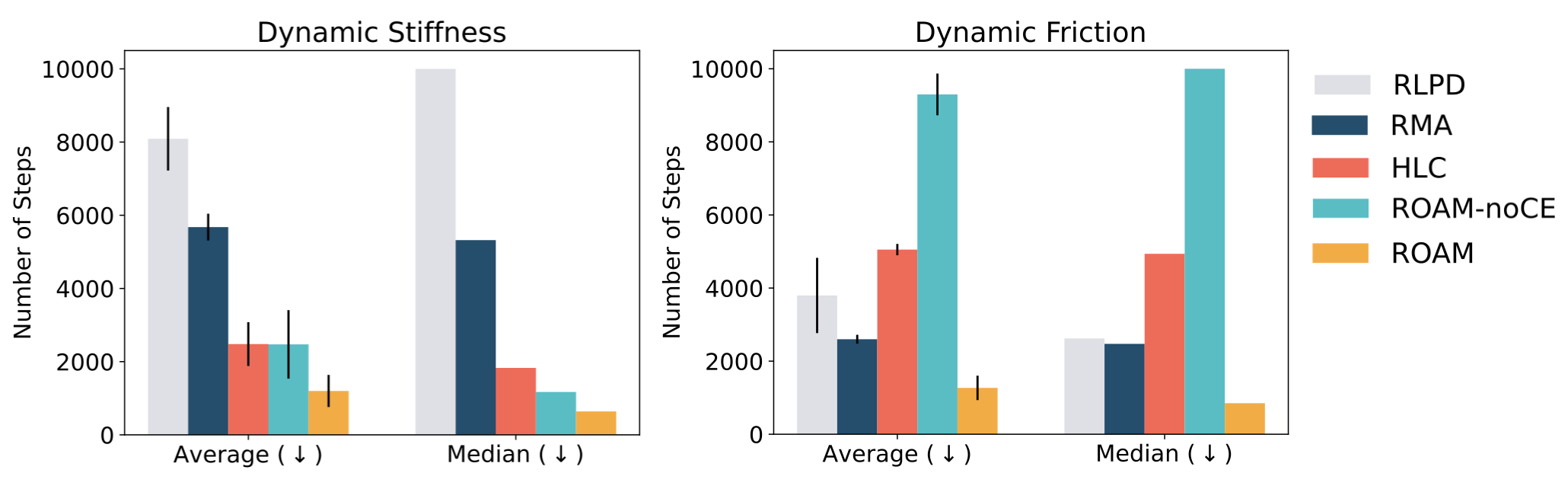}
    \caption{\small \textbf{Results on the simulated Go1 robot.} In both evaluation settings, \ours is over 2x as efficient as all comparisons in terms of both average and median number of steps taken to complete the task. }
    \label{fig:sim_results}
    \vspace{-4mm}
\end{figure*}

In this section, we evaluate the performance of \ours empirically and assess how effectively it can adapt on-the-fly to new situations.  
Concretely, we aim to answer the following questions: (1) In simulated and real-world settings, how does \ours compare to existing methods when given diverse prior behaviors/data and deployed in novel situations?
(2) How does the additional cross-entropy term in the loss function of \ours contribute to more successful utilization of the prior behaviors?
In the remainder of this section, we describe our experimental setup and present our results on both a simulated and a real-world Go1 quadrupedal robot. For qualitative video results, see our project webpage: 
{\footnotesize \url{https://sites.google.com/view/adapt-on-the-go/home}}.

\textbf{General Experimental Setup.}
We use the setting of locomotion deployment to evaluate \ours, as it is a challenging setting for adaptation, where the agent may naturally face a variety of different situations and must adapt its walking behavior on-the-fly without any additional supervision or human intervention.
We use a Go1 quadruped robot from Unitree and MuJoCo~\citep{todorov2012mujoco} for simulation. 
We implemented all methods on top of the same state-of-the art implementation of SAC from \citep{smith2022walk} as the base learning approach, with regularization additions following DroQ~\citep{hiraoka2021dropout}, and RLPD~\citep{ball2023rlpd} for methods that do online fine-tuning. 

\textbf{Comparisons.} We evaluate \ours along with the following prior methods: 
(1) RLPD Fine-tuning \citep{ball2023rlpd}, where we fine-tune a single policy using all the data from the prior behaviors with RLPD; 
(2) RMA \citep{kumar2021rma}, which trains a base policy and adaptation module that estimates environment info;
(3) High-level Classifier (HLC), which trains a classifier on the data buffers of the pre-trained behaviors and uses it to select which behavior to use at a given state, as a representative method for those that train an additional behavior selection network, similar to work by \citep{han2023lifelike}.
We additionally consider an ablation, \ours-NoCE, which uses the values of the prior behaviors to choose among behaviors but does not fine-tune with the additional cross-entropy loss.
We give RMA access to unlimited online episode rollouts in each of the prior MDPs during pre-training, while all other methods use the same set of offline data and prior behaviors that are pre-trained in the prior MDPs. 

\subsection{Selecting Relevant Behaviors in Simulation}

\textbf{Setup.} In our simulation experiments, we evaluate in two separate settings. The first setting simulates a situation where different joints become damaged or stuck during the robot's lifetime.
It uses 9 prior behaviors: each is a different limping behavior with a different joint frozen.
In the single life, the agent must walk a total distance of 10 meters, and every 100 steps, one of the 3 remaining joints \emph{not covered in the prior data} is frozen, and the agent must adapt its walking behavior to continue walking.
The second setting simulates a situation where the robot encounters different friction levels on its different feet due to variation in terrain. It uses 4 different prior behaviors, each of which is trained with one of the 4 feet having low friction. During the single life, every 50 steps, the friction of one or two of the feet is changed to be lower than in the prior behaviors.
To collect the prior behaviors, we train each behavior for 250k steps (first setting) or 50k steps (second setting) in the corresponding MDP, and use the last 40k steps as the prior data.
The agent is given a maximum of 10k steps to complete the task, and if it does not complete the task within this time, it is considered an unsuccessful trial.

\textbf{Results.} As seen in Figure~\ref{fig:sim_results}, \ours outperforms all other methods in all three metrics of average and median number of steps taken to complete the task.
In particular, in both settings, \ours completes the task \textit{more than 2 times faster}, in terms of average number of timesteps, compared to the next best method.
Both RLPD fine-tuning and RMA struggle on both evaluation settings, especially the stiffness setting, demonstrating the importance of adapting in the space of behaviors rather than the space of actions for more efficient adaptation.
RLPD and RMA perform better in the friction evaluation, where 
a single policy can still somewhat adapt to the various situations in the single life.
On the other hand, HLC and \ours-NoCE both struggle in the friction eval suite, demonstrating the importance of the additional cross-entropy term in the loss function, encouraging greater behavior specialization in different regions of the state space.
These two methods perform better when the behaviors are already more distinguished, as in the limping eval suite, but they still struggle to adapt as efficiently as \ours.

\begin{figure*}[t]
    \centering
    \begin{minipage}{0.485\textwidth}
        \centering
        \includegraphics[width=0.75\linewidth]{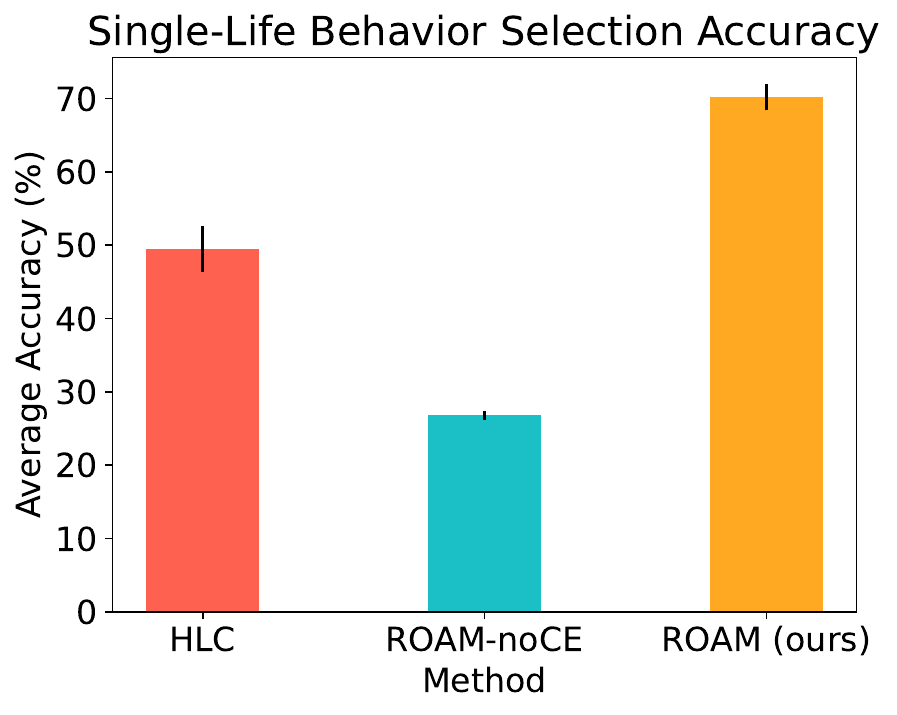}
        \caption{\small \textbf{Single-life Behavior Selection Accuracy.} The percent of steps where different methods select a relevant behavior for the current situation. 
        \ours is able to choose a relevant behavior significantly more often on average when adapting to test-time situations than HLC and \ours-noCE.}
        \label{fig:sub1}
    \end{minipage}
    \hfill
    \begin{minipage}{0.485\textwidth}
        \centering
        \includegraphics[width=0.95\linewidth]{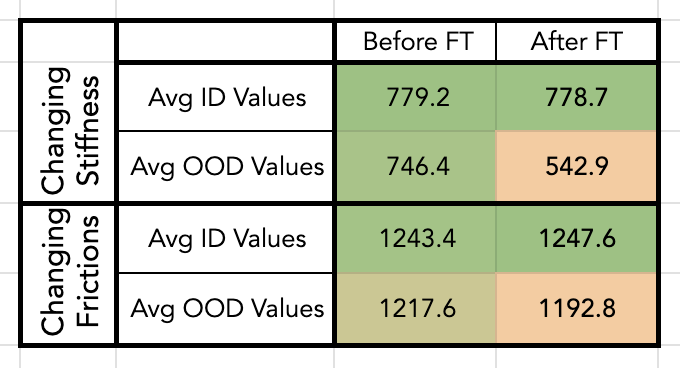}
        \caption{\small \textbf{Effect of the Cross-Entropy Loss on ID and OOD values.} The average values of the different behaviors in states visited by that behavior (ID) vs states visited by other behaviors (OOD), before and after fine-tuning with the additional cross-entropy term for \ours.
        \ours is able to maintain high values of the behavior in ID states, while decreasing the value of the behavior in OOD states.}
        \label{fig:sub2}
    \end{minipage}
\end{figure*}

In Figures~\ref{fig:sub1} and ~\ref{fig:sub2}, we provide some additional empirical analysis of \ours. First, in the stiffness evaluation, 
we plot the percent of steps where different methods select a relevant behavior during test-time deployment, where a held-out joint is frozen and a relevant behavior is one where an adjacent joint on the same leg is frozen or the same joint on an adjacent leg is frozen, and we see that \ours is able to choose the most relevant behavior on average significantly more frequently than HLC and \ours-noCE, which often select a behavior that is not relevant to the current situation.
In addition, we record the average values of the different behaviors in states visited by that behavior (ID) vs states visited by other behaviors (OOD), before and after fine-tuning with the additional cross-entropy term for \ours.
We find that \ours is able to effectively maintain high values of the behavior in familiar states, while decreasing the value of the behavior in unfamiliar OOD states.

\subsection{Adapting on-the-Go1 in the Real World}

\begin{figure*}[t]
    \centering
    \includegraphics[width=0.8\textwidth]{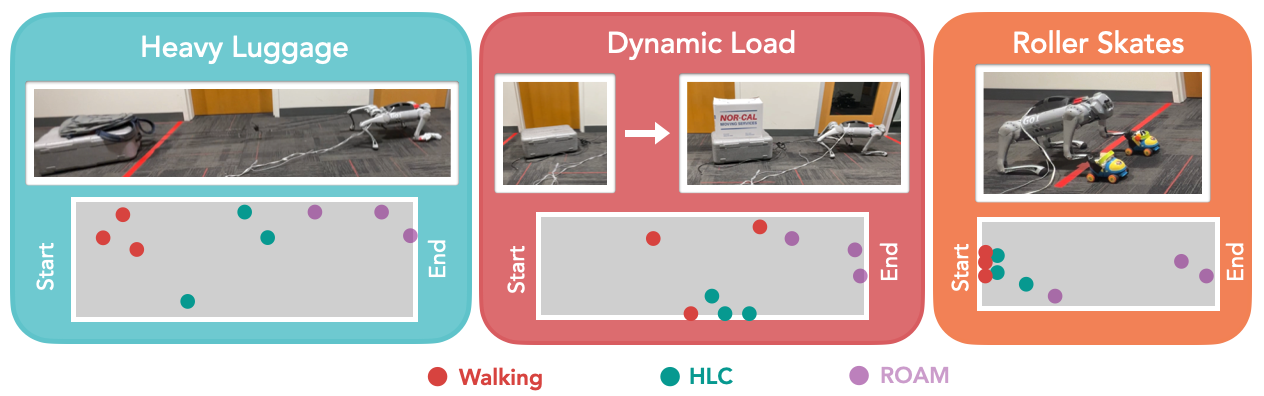}
    \caption{\small \textbf{Real-world single-life tasks.} We evaluate on: (1) pulling a load of heavy luggage 6.2 kg (13.6 lb), (2) pulling luggage where the weight dynamically changes between 2.36 kg (5.2 lb) and 4.2 kg (9.2 lb), and (3) moving forward with roller skates on the robot's front two feet. 
    For each trial, for each method, we also show the locations of the robot before its first fall or readjustment (Red is Walking, Blue is HLC, and Purple is \ours).}
    \label{fig:real-tasks}
\end{figure*}

\begin{table*}[t]
    \centering
    \adjustbox{max width=0.99\textwidth}{%
    \begin{tabular}{l|cc|cc|cc}
        \toprule
         & \multicolumn{2}{c}{Heavy Luggage} & \multicolumn{2}{c}{Dynamic Luggage Load} & \multicolumn{2}{c}{Roller Skates} \\
        \midrule
         & Avg. Time (s) $\downarrow$ & Falls $\downarrow$ & Avg. Time (s) $\downarrow$ & Falls $\downarrow$ & Avg. Time (s) $\downarrow$ & Falls $\downarrow$ \\
        \midrule
        Walking & 45.3 & 2.3 & 32 & 1 & NC & NC \\
        HLC & 42.7 & 3 & 28.3 & 1.3 & 62.3 & 2.7 \\
        \ours (ours) & \textbf{25.7} & \textbf{0.7} & \textbf{24.3} & \textbf{0.3} & \textbf{27.3} & \textbf{1} \\
        \bottomrule
    \end{tabular}
    }
    \caption{\small \textbf{Results on the real Go1 robot on 3 different tasks}: On all 3 tasks, across 3 trials for each method, \ours significantly outperforms both comparisons in terms of both average wall clock time (s) and number of falls or readjustments needed to complete the task in a single life. NC (no complete) indicates that the task was not able to be successfully completed with the given method. }
    \label{tab:real_results}
\end{table*}

\textbf{Setup.} On the Go1 quadruped robot, we evaluate \ours in a setting where we have a fixed set of five prior behaviors: walking, and four different behaviors where each of the legs has a joint frozen. 
We pre-train a base walking behavior in 18k steps and train the other behaviors by fine-tuning the walking behavior for an additional 3k steps with one of the joints frozen, all from scratch in the real world using the system from \citep{smith2022walk}.
During single-life deployment, we evaluate on the following three tasks: (1) Heavy Luggage: the robot must walk from a starting line to a finish line, while pulling a box that is 6.2 kg (13.6 lb) attached to one of the back legs. 
(2) Dynamic Luggage Load: the robot must walk from a starting line to a finish line, while adapting on-the-fly to a varying amount of weight between 2.36 kg (5.2 lb) and 4.2 kg (9.2 lb). We standardize each trial by adding and removing weight at the same distance from the start position. (3) Roller Skates: we fit the robot's front two feet into roller skates, and the robot must adapt to its behavior to slide its forward legs and push off its back legs in order to walk to the end line. 
We report the average wall clock time in seconds and number of falls or readjustments needed to complete the task in a single life across 3 trials for each method.
The tasks are shown in Figure~\ref{fig:real-tasks}, along with the locations of the robot before its first fall or readjustment for each method for each trial.

\textbf{Results.} Although none of the prior behaviors are trained to handle these specific test-time scenarios, the robot can leverage parts of the prior behaviors to complete the task.
As shown in Table~\ref{tab:real_results},\ours significantly outperforms using a high-level classifier (HLC) as well as the baseline walking policy in terms of both average wall clock time and number of falls or readjustments at single-life time on all three real-world tasks.
Qualitatively, the other methods have trouble pulling luggage consistently forward, whereas our method often chooses the behavior where a joint is frozen on the leg with the luggage attached, as this behavior uses the robot's other three legs to pull itself forward more effectively.
The other methods struggle particularly on the roller skates task, which has drastically different dynamics from typical walking and especially relies on choosing relevant behaviors that heavily use the back legs. 
As seen in Figure~\ref{fig:real-tasks}, for all three tasks, HLC and the standard walking policy often fall or need to be readjusted very early in each single-life trial, whereas \ours gets much closer to the finish line and often even completes the task without any falls or readjustments.

\subsection{Empirical Analysis of \ours}
\label{sec:app-emp-analysis}

The cross-entropy term is a regularizer that creates a preference for skills that visit a given state more frequently. However, this is not the only criterion for selecting a skill; it is a regularizer. A skill with higher value is still preferred if its visitation frequency is not too low, and ROAM does not exclusively always just select the most high-frequency behavior. We show this with the following experiment with results in Table~\ref{tab:chosen}. In the simulated stiffness suite, we held out most of the data from one of the buffers corresponding to one of the behaviors, leaving only 5k transitions compared to the original 40k, and evaluated the agent at test time in an environment suited for that behavior. We find that even with only 5k transitions (compared to 40k for all other behaviors), ROAM selects this less-frequent but suitable behavior the majority of the time, leading to similar overall performance.

\begin{table}[ht]
\centering
\begin{tabular}{ccc}
\toprule
\# Transitions & \% Timesteps Chosen & Avg \# Steps \\
\midrule
5k & 53.2 & 591.3 \\
40k & 78.4 & 573.8 \\
\bottomrule
\end{tabular}
\caption{\textbf{\ours selects high-value behaviors even with lower visitation frequency.} We find that even with a much smaller buffer, and therefore lower visitation frequency for many states, \ours still chooses that behavior when given situations suitable for it.}
\label{tab:chosen}
\end{table}

We next investigate the sensitivity of the $\beta$ hyperparameter. We ran ROAM with 4 different values (0.01, 0.1, 0.5, 0.9) of $\beta$ in each simulated suite and show the performance in Table~\ref{tab:betas}. For both evaluations, 3 out of 4 of these values (all except 0.01) outperform all the other baselines. 

\begin{table}[t]
\centering
\begin{tabular}{c|cc|cc}
\toprule
\multicolumn{1}{l|}{} & \multicolumn{2}{c|}{Dynamic Friction} & \multicolumn{2}{c}{Dynamic Stiffness} \\
$\beta$ & Avg \# Steps ($\downarrow$) & Frequency of Switching & Avg \# Steps ($\downarrow$) & Frequency of Switching \\
\toprule 
0.01 & 7610 $\pm$ 854 & 17.20\% & 2698 $\pm$ 844 & 2.92\% \\
0.1 & 2082 $\pm$ 382 & 15.63\% & 1331 $\pm$ 263 & 8.25\% \\
0.5 & 772 $\pm$ 179 & 11.85\% & 628 $\pm$ 19 & 12.35\% \\
0.9 & 1466 $\pm$ 534 & 9.36\% & 735 $\pm$ 54 & 13.36\% \\
\bottomrule
\end{tabular}
\caption{\textbf{Sensitivity of $\beta$ and frequency of behavior switching.} We find that a range of $\beta$ values give strong performance.}
\label{tab:betas}
\vspace{-4mm}
\end{table}

Additionally, one benefit of \ours is that the ability to switch between these policies at any timestep allows the agent to adapt to new and unforeseen situations, including those for which no single behavior is optimally suited. However, one hypothetical concern may be that frequent switching of behaviors may lead to suboptimal performance. In Table~\ref{tab:betas}, we measure how often behaviors were switched and tried to see if frequency of behavior switches correlates with failure. We found no such correlation. Below, we show the percent of timesteps where the agent decides to switch behaviors, and more frequent switching does not correlate to a higher average number of steps needed to complete the task.

At test time, \ours provides the option to use the collected online samples to further fine-tune the policies. We run an additional ablation on this aspect of fine-tuning in our simulated tasks. Below in Table~\ref{tab:ft-ablation}, we find that only performing fine-tuning (the RLPD comparison) leads to much slower adaptation. We also add a comparison to ROAM without low-level behavior fine-tuning (ROAM-selection only) in the table below, showing that fine-tuning the low-level behaviors helps but the lion’s share of the efficiency benefit comes from the behavior selection. Thus, the behavior selection mechanism is indeed the main driver of the fast adaptation, but our framework also supports fine-tuning the selected behaviors.

\begin{table}[t]
    \centering
    \begin{tabular}{lcc}
        \toprule
        & Dynamic Friction & Dynamic Stiffness \\
        \midrule
        & Avg \# Steps ($\downarrow$) & Avg \# Steps ($\downarrow$) \\
        \midrule
        RLPD & 8089 $\pm$ 868 & 3797 $\pm$ 1029 \\
        ROAM (selection only) & 1268 $\pm$ 334 & 1565 $\pm$ 339 \\
        ROAM & \textbf{1018 $\pm$ 225} & \textbf{1199 $\pm$ 439} \\
        \bottomrule
    \end{tabular}
    \caption{\textbf{Ablation on Behavior-Level Fine-Tuning.} \ours supports fine-tuning the selected behaviors with the additional collected data. The majority of efficiency benefits come from the behavior selection mechanism, though fine-tuning selected behaviors also improves performance.}
    \vspace{-4mm}
    \label{tab:ft-ablation}
\end{table}

\begin{table}[h]
    \centering
    \begin{tabular}{lcc}
        \toprule
        & Dynamic Friction & Dynamic Stiffness \\
        \midrule
        & Avg \# Steps ($\downarrow$) & Avg \# Steps ($\downarrow$) \\
        \midrule
        ROAM (original, selected skills) & 1018 $\pm$ 225 & 1199 $\pm$ 439 \\
        ROAM (all skills) & 1141 $\pm$ 115 & 871 $\pm$ 137 \\
        \bottomrule
    \end{tabular}
    \caption{\textbf{\ours with irrelevant behaviors.} Adding irrelevant behaviors to the set of prior behaviors does not significantly hurt performance and sometimes improves it, as seen in the stiffness evaluation.}
    \label{tab:skills}
    \vspace{-3mm}
\end{table}

Finally, in our main experiments in Section~\ref{sec:experiments}, we range from choosing among 4-9 different prior behaviors. We run an additional experiment combining all the behaviors from both simulated suites, giving 13 total, and in Table~\ref{tab:skills}, we find that for each suite, adding additional irrelevant behaviors did not significantly hurt performance and can even improve performance (as is the case for the stiffness eval). 

\section{Conclusion and Future Work}
\label{sec:conclusion}

We introduced Robust Autonomous Modulation (\ours), which enables agents to rapidly adapt to changing, out-of-distribution circumstances during deployment. Our contribution lies in offering a principled, efficient way for agents to leverage pre-trained behaviors when adapting on-the-fly. 
On simulated and complex real-world tasks with a Go1 quadruped robot, our method achieves over 2x efficiency in adapting to new situations compared to existing methods. 
While \ours offers significant advances in enabling agents to adapt to out-of-distribution scenarios, one current limitation lies in the dependency on the range of pre-trained behaviors; some scenarios may simply be too far out-of-distribution compared to the available prior behaviors. 
Future work could explore integrating \ours into a lifelong learning framework, allowing agents to continuously expand their repertoire of behaviors and increasing their adaptability to more unforeseen situations.

\newpage
\bibliography{references}
\bibliographystyle{collas2025_conference}

\appendix
\section{Appendix}
\label{sec:app}

\subsection{Theoretical Analysis}
\label{sec:app-theory}

We present the proof of the following theorem, discussed in Section~\ref{sec:method}. 

\begin{theorem}
    Let $p_i(s)$ denote the state visitation probability for a behavior $b_i$ at state $s$. For any state $s$ that is out of distribution for behavior $b_i$ and is in distribution for another behavior $b_j$, i.e. $p_i(s) \ll p_j(s)$, if $1 > \beta > 0$ is chosen to be large enough, then the value of behavior $b_i$ learned by \ours will be bounded above compared to value of behavior $b_j$, i.e.  $V_i(s) \leq V_j(s)$.
    \label{thm-main}
\end{theorem}

\begin{proof} 

We optimize the following loss function: 
\begin{align*}
\mc{L}_{\text{fine-tune}} & = 
 (1 - \beta) \sum_i \sum_{s \sim \mc{D}_i} p_i(s)\left[ \left( R_i(s) + \gamma \E_{s'} V_i(s') - V_i(s) \right)^2 \right]
+ \beta \sum_{j} \sum_{s \sim \D_j} p_j(s) \left[ -V_j(s) + \log \sum_{k=1}^n \exp(V_k(s)) \right].
\end{align*}

Taking the derivative with respect to $V_i(s)$, we have:
\begin{align*}
\frac{\partial\mc{L}_{\text{fine-tune}}}{\partial V_i(s)} & = 2(1 - \beta)p_i(s) (\gamma p(s|s, a) - 1)\left(R_i(s) + \gamma \E_{s'} V_i(s') - V_i(s) \right) \\
&+ \beta \left[ p_i(s)\left(-1 + \frac{\exp(V_i(s))}{\sum_{k=1}^n \exp(V_k(s))}\right) + \sum_{j \neq i} p_j(s) \frac{\exp(V_i(s))}{\sum_{k=1}^n \exp(V_k(s))}\right] \\
&= 2(1 - \beta)p_i(s) (\gamma P_i(s|s, a) - 1)\left( R_i(s) + \gamma \E_{s'} V_i(s') - V_i(s) \right) 
+ \beta \left[-p_i(s) + \sum_{j} p_j(s) \frac{\exp(V_i(s))}{\sum_{k=1}^n \exp(V_k(s))}\right].
\end{align*}

Setting to 0, we have the following characterization of any convergence point:
\begin{align*}
V_i(s) & = \left( R_i(s) + \gamma \E_{s'} V_i(s')\right) 
+ \frac{\beta}{2(1 - \beta)p_i(s) (\gamma P_i(s|s, a) - 1)} \left[-p_i(s) + \sum_{j} p_j(s) \frac{\exp(V_i(s))}{\sum_{k=1}^n \exp(V_k(s))}\right] \\
&= \left( R_i(s) + \gamma \E_{s'} V_i(s')\right) 
+ \frac{\beta}{2(1 - \beta)(1 - \gamma P_i(s|s, a))} \left[1 - \sum_{j} \frac{p_j(s)}{p_i(s)} \frac{\exp(V_i(s))}{\sum_{k=1}^n \exp(V_k(s))}\right].
\end{align*}

Consider a state s where $p_{\text{freq}}(s) \gg p_i(s)$ for some behavior $b_{\text{freq}} \neq b_i$. We want to show that in such a state $V_{freq}(s) > V_i(s)$. 
Then $\frac{p_{\text{freq}}(s)}{p_i(s)}$ will dominate in the last term, so 
\begin{equation}
    \sum_{j} \frac{p_j(s)}{p_i(s)} \frac{\exp(V_i(s))}{\sum_{k=1}^n \exp(V_k(s))} >> 1.
    \label{eq:assumption}
\end{equation}

Then comparing the convergence point values of $V_i(s)$ and $V_{\text{freq}}(s)$, we have
\begin{align*}
    V_{\text{freq}}(s) - V_i(s) &= \left( R_{\text{freq}}(s) + \gamma \E_{s'} V_{\text{freq}}(s')\right) 
+ \frac{\beta}{2(1 - \beta)(1 - \gamma P_{\text{freq}}(s|s, a))} \left[1 - \sum_{j} \frac{p_j(s)}{p_{\text{freq}}(s)} \frac{\exp(V_{\text{freq}}(s))}{\sum_{k=1}^n \exp(V_k(s))}\right] \\
&- \left[\left( R_i(s) + \gamma \E_{s'} V_i(s')\right) 
+ \frac{\beta}{2(1 - \beta)(1 - \gamma P_i(s|s, a))} \left[1 - \sum_{j} \frac{p_j(s)}{p_i(s)} \frac{\exp(V_i(s))}{\sum_{k=1}^n \exp(V_k(s))}\right] \right] \\
&= (\left( R_{\text{freq}}(s) + \gamma \E_{s'} V_{\text{freq}}(s')\right) - \left( R_i(s) + \gamma \E_{s'} V_i(s')\right)) \\
&+ \frac{\beta}{1 - \beta} \left(C_{\text{freq}} \left[1 - \sum_{j} \frac{p_j(s)}{p_{\text{freq}}(s)} \frac{\exp(V_{\text{freq}}(s))}{\sum_{k=1}^n \exp(V_k(s))}\right] - C_i\left[1 - \sum_{j} \frac{p_j(s)}{p_i(s)} \frac{\exp(V_i(s))}{\sum_{k=1}^n \exp(V_k(s))}\right] \right),
\end{align*}
where $C_{\text{freq}} = \frac{1}{2(1 - \gamma P_{\text{freq}}(s|s, a))} > 0$ and $C_i = \frac{1}{2(1 - \gamma P_i(s|s, a))} > 0$.
Thus, using Equation~\ref{eq:assumption}, because  
\[1 - \sum_{j} \frac{p_j(s)}{p_i(s)} \frac{\exp(V_i(s))}{\sum_{k=1}^n \exp(V_k(s))} << 0,
\]
the term 
\[\left(C_{\text{freq}} \left[1 - \sum_{j} \frac{p_j(s)}{p_{\text{freq}}(s)} \frac{\exp(V_{\text{freq}}(s))}{\sum_{k=1}^n \exp(V_k(s))}\right] - C_i\left[1 - \sum_{j} \frac{p_j(s)}{p_i(s)} \frac{\exp(V_i(s))}{\sum_{k=1}^n \exp(V_k(s))}\right] \right) > 0. \]
Hence, for some $0 < \beta < 1$, we have
\begin{align*}
V_{\text{freq}}(s) - V_i(s) &=(\left( R_{\text{freq}}(s) + \gamma \E_{s'} V_{\text{freq}}(s')\right) - \left( R_i(s) + \gamma \E_{s'} V_i(s')\right)) \\
&+ \frac{\beta}{1 - \beta} \left(C_{\text{freq}} \left[1 - \sum_{j} \frac{p_j(s)}{p_{\text{freq}}(s)} \frac{\exp(V_{\text{freq}}(s))}{\sum_{k=1}^n \exp(V_k(s))}\right] - C_i\left[1 - \sum_{j} \frac{p_j(s)}{p_i(s)} \frac{\exp(V_i(s))}{\sum_{k=1}^n \exp(V_k(s))}\right] \right) > 0.
\end{align*}
Thus, for some for some $0 < \beta < 1$, we have $V_i(s) < V_{\text{freq}}(s)$.

\end{proof}

To illustrate the effect of the cross-entropy loss, consider the following example of choosing between two behaviors $i$ and $j$ at state $s$, where the true values $V^{\text{true}}_i(s) < V^{\text{true}}_j(s)$. The optimal choice is to choose behavior $b_j$ and for sake of this example let us choose behavior $b_j$ if $V^{\text{\ours}}_i(s) < V^{\text{\ours}}_j(s)$. There are the following four cases: (1) $p_i(s) < p_j(s)$ and the initial estimated $V_i(s) < V_j(s)$. Then with any $\beta > 0$, the final $V^{\text{\ours}}_i(s) < V^{\text{\ours}}_j(s)$; (2) $p_i(s) < p_j(s)$ and the initial estimated $V_i(s) > V_j(s)$. Then by Theorem~\ref{thm-main}, with large enough $\beta > 0$, the final $V^{\text{\ours}}_i(s) < V^{\text{\ours}}_j(s)$; (3) $p_i(s) > p_j(s)$ and the initial estimated $V_i(s) < V_j(s)$. Then as long as $\beta$ is chosen to be not too large, the final $V^{\text{\ours}}_i(s) < V^{\text{\ours}}_j(s)$. (4) $p_i(s) > p_j(s)$ and the initial estimated $V_i(s) > V_j(s)$. This is the only case where \ours may not be adjusted to work well, but this case poses a difficult situation for any behavior selection method. 

\subsection{Algorithm Summary}
\label{sec:app-method}
We summarize \ours in Algorithms~\ref{algoblock1} and ~\ref{algoblock2}.

\input{corl_sections/algorithm}

\subsection{General Experiment Setup Details}
\label{sec:app-expsetup}

As common practice in learning-based quadrupedal locomotion works, we define actions to be PD position targets for the 12 joints, and we use a control frequency of 20 Hz. Actions are centered around the nominal pose, i.e. 0 is standing. We describe the observations for the simulated and real-world experiments below.

We first detail the reward function we use to define the quadrupedal walking task. First, we have a velocity-tracking term defined as follows:
\[
    r_v(s, a) = 1 - |\frac{v_x - v_t}{v_t}|^{1.6}
\]
where $v_t$ is the target velocity and $v_x$ is the robot's local, forward linear velocity projected onto the ground plane, i.e., $v_x = v_\text{local} \cdot \cos(\phi)$ where $\phi$ is the root body's pitch angle. We then have a term $r_{ori}(s, a)$ that encourages the robot to stay upright. Specifically, we calculate the cosine distance between the 3d vector perpendicular to the robot's body and the gravity vector ($[0, 0, 1]$). We then normalize the term to be between 0 and 1 via:
\[
    r_{ori}(s, a) = (0.5\cdot\texttt{dist} + 0.5)^2
\]
where $\texttt{dist}$ is the cosine distance as described above.
We multiply $r_v$ and $r_{ori}$ so as to give reward for tracking velocity proportionally to how well the robot is staying upright.
We then have a regularization term $r_{qpos}$ to favor solutions that are close to the robot's nominal standing pose. This regularization term is calculated as a product of a normalized term per-joint. Below, $\hat{q}^j$ represents the local rotation of joint $j$ of the nominal pose, and $q^j$ represents the robot's joint,
\[
    r^\text{qpos} = 1 - \prod_j \texttt{q distance}(\hat{q}^j, q^j)
\]
where $\texttt{q distance}$ is between 0 and 1 and decays quadratically until a threshold which is the robot's action limits. Specifically, we follow the reward structure put forth by \citep{tunyasuvunakool2020}.
These terms comprise the overwhelming majority of weight in the final reward. We also include terms for avoiding undesirable behaviors like rocking or swaying that penalize any angular velocity in the root body's roll, pitch, and yaw. We also slightly penalize energy consumption and torque smoothness. To encourage a walking gait in particular, we added another regularization term to encourage diagonal shoulder and hip joints to be the same at any given time.

\subsection{Implementation Details and Hyperparameters}
\label{sec:app-hyperparams}
We implemented all methods, including ROAM, RMA, and HLC, on top of the same state-of-the art implementation of SAC from \citep{smith2022walk} as the base learning approach. For all comparisons, we additionally use a high UTD ratio, dropout, and layernorm, following DroQ~\citep{hiraoka2021dropout}, and for methods that do online fine-tuning, we use 50/50 sampling following RLPD~\citep{ball2023rlpd}.
We use default hyperparameter values: a learning rate of $3 \times 10^{-4}$, an online batch size of $128$, and a discount factor of $0.99$.
The policy and critic networks are MLPs with 2 hidden layers of 256 units each and ReLU activations. For \ours, we tuned $\beta$ with values 0.01, 0.1, 0.5, 0.9.

\paragraph{Simulated Experiments.}
For the simulated experiments, the state space consists of joint positions, joint velocities, torques, IMU (roll, pitch, change in roll, change in pitch), and normalized foot forces for a total of $44$ dimensions. For the position controller, we use $K_p$ and $K_d$ gains of $40$ and $5$, respectively, and calculate torques for linearly interpolated joint angles from current to desired at 500Hz. We define the limits of the action space to be 30\% of the physical joint limits.

\begin{table}[h]
    \centering
    \caption{Simulated Reward Function Parameter Details}
    \begin{tabular}{|l|l|}
        \hline
        \multicolumn{1}{|c|}{Parameter} & \multicolumn{1}{c|}{Value} \\
        \hline
        Target Velocity & 1.0 \\
        Energy Penalty Weight & 0.008 \\
        Qpos Penalty Weight & 10.0 \\
        Smooth Torque Penalty Weight & 0.005 \\
        Pitch Rate Penalty Factor & 0.6 \\
        Roll Rate Penalty Factor & 0.6 \\
        Joint Diagonal Penalty Weight & 0.1 \\
        Joint Shoulder Penalty Weight & 0.15 \\
        Smooth Change in Target Delta Yaw Steps & 5 \\
        \hline
    \end{tabular}
\end{table}

For the first experimental setting, we train prior behavior policies with high stiffness ($10.0$) in $9$ different individual joints. Specifically, we use the front right body joint, the front right knee joint, the front left body joint, the front left knee joint, the rear right body joint, the rear right knee joint, the rear left body joint, the rear left thigh joint, and the rear left knee joint. During deployment, we switch between $3$ conditions every $100$ steps. Condition 1 is applying stiffness $15.0$ to the rear right thigh joint, condition 2 is applying stiffness $15.0$ to the front left thigh joint, and condition 3 is applying stiffness $15.0$ to the front right thigh joint. For this setting, we use $\beta=0.5$ for \ours.

For the second experimental setting, we train prior behavior policies with low foot friction ($0.4$) in each of the $4$ feet. During deployment, we switch between $2$ conditions every $50$ steps. Condition 1 is applying a foot friction of $0.1$ to the rear right foot and condition 2 is applying a foot friction of $0.01$ to the front left foot and a foot friction of $0.1$ too the rear right foot. For this setting, we use $\beta=0.5$ for \ours.

\paragraph{Real-world Experiments.}
For the real-world experiments, the state space consists of joint positions, joint velocities, torques, forward linear velocity, IMU (roll, pitch, change in roll, change in pitch), and normalized foot forces for a total of $47$ dimensions. We use an Intel T265 camera-based velocity estimator to estimate onboard linear velocity. We use $K_p$ and $K_d$ gains of $20$ and $1$, respectively, which are used in the position controller. We again use action interpolation, an action range of 35\% physical limits, and a $1$ step action history. We also use a second-order Butterworth low-pass filter with a high-cut value of $8$ to smooth the position targets. Finally, to reset the robot, we use the reset policy provided by \citep{smith2022walk}. We train $4$ prior behavior policies for the real-world experiments, each of which is trained with a frozen knee joint. Specifically, we train a policy with the front right knee joint frozen, the front left knee joint frozen, the rear right knee joint frozen, and the rear left knee joint frozen. $\beta=0.5$ is used in all real-world experiments for \ours.

\begin{table}[h]
    \centering
    \caption{Real-world Reward Function Parameter Details}
    \begin{tabular}{|l|l|}
        \hline
        \multicolumn{1}{|c|}{Parameter} & \multicolumn{1}{c|}{Value} \\
        \hline
        Target Velocity & 1.5 \\
        Energy Penalty Weight & 0.0 \\
        Qpos Penalty Weight & 2.0 \\
        Smooth Torque Penalty Weight & 0.005 \\
        Pitch Rate Penalty Factor & 0.4 \\
        Roll Rate Penalty Factor & 0.2 \\
        Joint Diagonal Penalty Weight & 0.03 \\
        Joint Shoulder Penalty Weight & 0.0 \\
        Smooth Change in Target Delta Yaw Steps & 1 \\
        \hline
    \end{tabular}
\end{table}

\paragraph{HLC Details.}
For HLC, we have an MLP that takes state as input and outputs which behavior to select in the given state. The MLP has 3 hidden layers of $256$ units each and ReLU activations, and we train by sampling from the combined offline data from all prior behaviors. We use a batch size of $256$, learning rate of $3 \times 10^{-4}$, and train for $3,000$ iterations.

\paragraph{RMA Details.}
For RMA training, we changed the environment dynamics between each episode and trained for a total of $2,000,000$ iterations. The standard architecture and hyperparameter choices from \citep{kumar2021rma} were used along with DroQ~\citep{hiraoka2021dropout} as the base algorithm.

\end{document}

%% file: math_commands.tex
\usepackage{amsmath,amsfonts,bm}

\def\eqref#1{equation~\ref{#1}}

\def\1{\bm{1}}

\DeclareMathAlphabet{\mathsfit}{\encodingdefault}{\sfdefault}{m}{sl}
\SetMathAlphabet{\mathsfit}{bold}{\encodingdefault}{\sfdefault}{bx}{n}

\newcommand{\E}{\mathbb{E}}

%% file: corl_sections/algorithm.tex
\begin{figure*}[h]
\begin{minipage}[t]{0.4\textwidth}
\begin{algorithm}[H]
\caption{\small \textsc{\ours Fine-Tuning}}
\begin{algorithmic}[1]
\STATE \textbf{Require:} $\mc{D}_i$, pre-trained critics $Q_{i, \text{orig}}$
\WHILE{\text{not converged}}
    \FORALL{$i$ in 1, ..., $N_\text{behaviors}$}
    \STATE Sample $(s, a, s', r_{\text{target}}) \sim D_i$
    \STATE Update $Q_i$ according to Eq.~\ref{eq:pretrain}
    \ENDFOR
\ENDWHILE
\STATE return $Q_1, ..., Q_{N_\text{behaviors}}$
\end{algorithmic}
\label{algoblock1}
\end{algorithm}
\end{minipage}
\noindent
\hfill
\begin{minipage}[t]{0.59\textwidth}
\begin{algorithm}[H]
\caption{\small \textsc{\ours Single-Life Deployment}}
\begin{algorithmic}[1]
\STATE \textbf{Require:} Test MDP $\mc{M}_{\text{test}}$, $\mc{D}_i$, policies $\pi_i$ and fine-tuned critics $Q_i$; 
\STATE \textbf{Initialize:} online replay buffers $\mc{D}^i_{\text{online}}$; timestep $t=0$
\WHILE{\text{task not complete}}
    \STATE Compute values of each behavior $\{V_i(s_t)\}_1^{N_\text{behaviors}}$ 
    \STATE Sample behavior $b^*$ according to the distribution softmax($\exp(V_i(s_t))$).
    \STATE Take action $a_t \sim \pi_{b^*}(a_t | s_t)$.
    \STATE $\mc{D}^{b^*}_{\text{online}} \leftarrow \mc{D}^{b^*}_{\text{online}} \cup \{(s_t, a_t, r_t, s_{t+1})\}$ 
    \STATE $Q_{b^*, \text{orig}}(s, a), \pi_{b^*} \leftarrow \textsc{RL}(Q_{b^*, \text{orig}}(s, a), \pi_{b^*}, \mc{D}^{b^*}_{\text{online}})$
    \STATE Increment $t$
\ENDWHILE{}
\end{algorithmic}
\label{algoblock2}
\end{algorithm}
\end{minipage}
\vspace{-.25cm}
\end{figure*}

%% file: collas2025_conference.bbl
\begin{thebibliography}{70}
\providecommand{\natexlab}[1]{#1}
\providecommand{\url}[1]{\texttt{#1}}
\expandafter\ifx\csname urlstyle\endcsname\relax
  \providecommand{\doi}[1]{doi: #1}\else
  \providecommand{\doi}{doi: \begingroup \urlstyle{rm}\Url}\fi

\bibitem[Achiam et~al.(2018)Achiam, Edwards, Amodei, and Abbeel]{achiam2018variational}
Joshua Achiam, Harrison Edwards, Dario Amodei, and Pieter Abbeel.
\newblock Variational option discovery algorithms.
\newblock \emph{arXiv preprint arXiv:1807.10299}, 2018.

\bibitem[Agarwal et~al.(2022)Agarwal, Kumar, Malik, and Pathak]{agarwal2023legged}
Ananye Agarwal, Ashish Kumar, Jitendra Malik, and Deepak Pathak.
\newblock Legged locomotion in challenging terrains using egocentric vision.
\newblock In \emph{Conference on Robot Learning}, 2022.

\bibitem[Akkaya et~al.(2019)Akkaya, Andrychowicz, Chociej, Litwin, McGrew, Petron, Paino, Plappert, Powell, Ribas, et~al.]{akkaya2019solving}
Ilge Akkaya, Marcin Andrychowicz, Maciek Chociej, Mateusz Litwin, Bob McGrew, Arthur Petron, Alex Paino, Matthias Plappert, Glenn Powell, Raphael Ribas, et~al.
\newblock Solving rubik's cube with a robot hand.
\newblock \emph{arXiv preprint arXiv:1910.07113}, 2019.

\bibitem[Bacon et~al.(2017)Bacon, Harb, and Precup]{bacon2017option}
Pierre-Luc Bacon, Jean Harb, and Doina Precup.
\newblock The option-critic architecture.
\newblock In \emph{Proceedings of the AAAI conference on artificial intelligence}, volume~31, 2017.

\bibitem[Ball et~al.(2023)Ball, Smith, Kostrikov, and Levine]{ball2023rlpd}
Philip~J Ball, Laura Smith, Ilya Kostrikov, and Sergey Levine.
\newblock Efficient online reinforcement learning with offline data.
\newblock \emph{arXiv preprint arXiv:2302.02948}, 2023.

\bibitem[Baumli et~al.(2021)Baumli, Warde-Farley, Hansen, and Mnih]{baumli2021relative}
Kate Baumli, David Warde-Farley, Steven Hansen, and Volodymyr Mnih.
\newblock Relative variational intrinsic control.
\newblock In \emph{Proceedings of the AAAI conference on artificial intelligence}, volume~35, pp.\  6732--6740, 2021.

\bibitem[Chen et~al.(2022)Chen, Sharma, Levine, and Finn]{chen2022you}
Annie Chen, Archit Sharma, Sergey Levine, and Chelsea Finn.
\newblock You only live once: Single-life reinforcement learning.
\newblock \emph{Advances in Neural Information Processing Systems}, 35:\penalty0 14784--14797, 2022.

\bibitem[Chitnis et~al.(2020)Chitnis, Tulsiani, Gupta, and Gupta]{chitnis2020efficient}
Rohan Chitnis, Shubham Tulsiani, Saurabh Gupta, and Abhinav Gupta.
\newblock Efficient bimanual manipulation using learned task schemas.
\newblock In \emph{2020 IEEE International Conference on Robotics and Automation (ICRA)}, pp.\  1149--1155. IEEE, 2020.

\bibitem[Cully et~al.(2015)Cully, Clune, Tarapore, and Mouret]{cully2015robots}
Antoine Cully, Jeff Clune, Danesh Tarapore, and Jean-Baptiste Mouret.
\newblock Robots that can adapt like animals.
\newblock \emph{Nature}, 521\penalty0 (7553):\penalty0 503--507, 2015.

\bibitem[Cutler et~al.(2014)Cutler, Walsh, and How]{cutler2014reinforcement}
Mark Cutler, Thomas~J Walsh, and Jonathan~P How.
\newblock Reinforcement learning with multi-fidelity simulators.
\newblock In \emph{2014 IEEE International Conference on Robotics and Automation (ICRA)}, pp.\  3888--3895. IEEE, 2014.

\bibitem[Dalal et~al.(2021)Dalal, Pathak, and Salakhutdinov]{dalal2021accelerating}
Murtaza Dalal, Deepak Pathak, and Russ~R Salakhutdinov.
\newblock Accelerating robotic reinforcement learning via parameterized action primitives.
\newblock \emph{Advances in Neural Information Processing Systems}, 34:\penalty0 21847--21859, 2021.

\bibitem[Duan et~al.(2016)Duan, Schulman, Chen, Bartlett, Sutskever, and Abbeel]{duan2016rl}
Yan Duan, John Schulman, Xi~Chen, Peter~L Bartlett, Ilya Sutskever, and Pieter Abbeel.
\newblock Rl $^2$: Fast reinforcement learning via slow reinforcement learning.
\newblock \emph{arXiv preprint arXiv:1611.02779}, 2016.

\bibitem[Eysenbach et~al.(2018)Eysenbach, Gupta, Ibarz, and Levine]{eysenbach2018diversity}
Benjamin Eysenbach, Abhishek Gupta, Julian Ibarz, and Sergey Levine.
\newblock Diversity is all you need: Learning skills without a reward function.
\newblock \emph{arXiv preprint arXiv:1802.06070}, 2018.

\bibitem[Eysenbach et~al.(2020)Eysenbach, Asawa, Chaudhari, Levine, and Salakhutdinov]{eysenbach2020off}
Benjamin Eysenbach, Swapnil Asawa, Shreyas Chaudhari, Sergey Levine, and Ruslan Salakhutdinov.
\newblock Off-dynamics reinforcement learning: Training for transfer with domain classifiers.
\newblock \emph{arXiv preprint arXiv:2006.13916}, 2020.

\bibitem[Finn et~al.(2017)Finn, Abbeel, and Levine]{finn2017model}
Chelsea Finn, Pieter Abbeel, and Sergey Levine.
\newblock Model-agnostic meta-learning for fast adaptation of deep networks.
\newblock In \emph{International conference on machine learning}, pp.\  1126--1135. PMLR, 2017.

\bibitem[Fu et~al.(2022)Fu, Cheng, and Pathak]{fu2023deep}
Zipeng Fu, Xuxin Cheng, and Deepak Pathak.
\newblock Deep whole-body control: learning a unified policy for manipulation and locomotion.
\newblock In \emph{Conference on Robot Learning}, 2022.

\bibitem[Gregor et~al.(2016)Gregor, Rezende, and Wierstra]{gregor2016variational}
Karol Gregor, Danilo~Jimenez Rezende, and Daan Wierstra.
\newblock Variational intrinsic control.
\newblock \emph{arXiv preprint arXiv:1611.07507}, 2016.

\bibitem[Haarnoja et~al.(2018)Haarnoja, Zhou, Abbeel, and Levine]{haarnoja2018soft}
Tuomas Haarnoja, Aurick Zhou, Pieter Abbeel, and Sergey Levine.
\newblock Soft actor-critic: Off-policy maximum entropy deep reinforcement learning with a stochastic actor.
\newblock In \emph{International conference on machine learning}, pp.\  1861--1870. PMLR, 2018.

\bibitem[Haarnoja et~al.(2023)Haarnoja, Moran, Lever, Huang, Tirumala, Wulfmeier, Humplik, Tunyasuvunakool, Siegel, Hafner, et~al.]{haarnoja2023learning}
Tuomas Haarnoja, Ben Moran, Guy Lever, Sandy~H Huang, Dhruva Tirumala, Markus Wulfmeier, Jan Humplik, Saran Tunyasuvunakool, Noah~Y Siegel, Roland Hafner, et~al.
\newblock Learning agile soccer skills for a bipedal robot with deep reinforcement learning.
\newblock \emph{arXiv preprint arXiv:2304.13653}, 2023.

\bibitem[Han et~al.(2023)Han, Zhu, Sheng, Zhang, Li, Zhang, Zhang, Liu, Zhou, Zhao, et~al.]{han2023lifelike}
Lei Han, Qingxu Zhu, Jiapeng Sheng, Chong Zhang, Tingguang Li, Yizheng Zhang, He~Zhang, Yuzhen Liu, Cheng Zhou, Rui Zhao, et~al.
\newblock Lifelike agility and play on quadrupedal robots using reinforcement learning and generative pre-trained models.
\newblock \emph{arXiv preprint arXiv:2308.15143}, 2023.

\bibitem[Hansen et~al.(2020)Hansen, Jangir, Sun, Aleny{\`a}, Abbeel, Efros, Pinto, and Wang]{hansen2020self}
Nicklas Hansen, Rishabh Jangir, Yu~Sun, Guillem Aleny{\`a}, Pieter Abbeel, Alexei~A Efros, Lerrel Pinto, and Xiaolong Wang.
\newblock Self-supervised policy adaptation during deployment.
\newblock \emph{arXiv preprint arXiv:2007.04309}, 2020.

\bibitem[Hiraoka et~al.(2021)Hiraoka, Imagawa, Hashimoto, Onishi, and Tsuruoka]{hiraoka2021dropout}
Takuya Hiraoka, Takahisa Imagawa, Taisei Hashimoto, Takashi Onishi, and Yoshimasa Tsuruoka.
\newblock Dropout q-functions for doubly efficient reinforcement learning.
\newblock \emph{arXiv preprint arXiv:2110.02034}, 2021.

\bibitem[Houthooft et~al.(2018)Houthooft, Chen, Isola, Stadie, Wolski, Jonathan~Ho, and Abbeel]{houthooft2018evolved}
Rein Houthooft, Yuhua Chen, Phillip Isola, Bradly Stadie, Filip Wolski, OpenAI Jonathan~Ho, and Pieter Abbeel.
\newblock Evolved policy gradients.
\newblock \emph{Advances in Neural Information Processing Systems}, 31, 2018.

\bibitem[Ji et~al.(2022)Ji, Mun, Kim, and Hwangbo]{ji2022concurrent}
Gwanghyeon Ji, Juhyeok Mun, Hyeongjun Kim, and Jemin Hwangbo.
\newblock Concurrent training of a control policy and a state estimator for dynamic and robust legged locomotion.
\newblock \emph{IEEE Robotics and Automation Letters}, 2022.

\bibitem[Julian et~al.(2020)Julian, Swanson, Sukhatme, Levine, Finn, and Hausman]{julian2020never}
Ryan Julian, Benjamin Swanson, Gaurav Sukhatme, Sergey Levine, Chelsea Finn, and Karol Hausman.
\newblock Never stop learning: The effectiveness of fine-tuning in robotic reinforcement learning.
\newblock 2020.
\newblock URL \url{https://arxiv.org/abs/2004.10190}.

\bibitem[Khetarpal et~al.(2020)Khetarpal, Riemer, Rish, and Precup]{khetarpal2020towards}
Khimya Khetarpal, Matthew Riemer, Irina Rish, and Doina Precup.
\newblock Towards continual reinforcement learning: A review and perspectives.
\newblock \emph{arXiv preprint arXiv:2012.13490}, 2020.

\bibitem[Kumar et~al.(2021)Kumar, Fu, Pathak, and Malik]{kumar2021rma}
Ashish Kumar, Zipeng Fu, Deepak Pathak, and Jitendra Malik.
\newblock Rma: Rapid motor adaptation for legged robots.
\newblock \emph{arXiv preprint arXiv:2107.04034}, 2021.

\bibitem[Kumar et~al.(2020)Kumar, Zhou, Tucker, and Levine]{kumar2020conservative}
Aviral Kumar, Aurick Zhou, George Tucker, and Sergey Levine.
\newblock Conservative q-learning for offline reinforcement learning.
\newblock \emph{Advances in Neural Information Processing Systems}, 33:\penalty0 1179--1191, 2020.

\bibitem[Laskin et~al.(2022)Laskin, Liu, Peng, Yarats, Rajeswaran, and Abbeel]{laskin2022cic}
Michael Laskin, Hao Liu, Xue~Bin Peng, Denis Yarats, Aravind Rajeswaran, and Pieter Abbeel.
\newblock Cic: Contrastive intrinsic control for unsupervised skill discovery.
\newblock \emph{arXiv preprint arXiv:2202.00161}, 2022.

\bibitem[Lee et~al.(2020)Lee, Hwangbo, Wellhausen, Koltun, and Hutter]{lee2020learning}
Joonho Lee, Jemin Hwangbo, Lorenz Wellhausen, Vladlen Koltun, and Marco Hutter.
\newblock Learning quadrupedal locomotion over challenging terrain.
\newblock \emph{Science robotics}, 5\penalty0 (47):\penalty0 eabc5986, 2020.

\bibitem[Lee et~al.(2019)Lee, Yang, and Lim]{lee2019learning}
Youngwoon Lee, Jingyun Yang, and Joseph~J Lim.
\newblock Learning to coordinate manipulation skills via skill behavior diversification.
\newblock In \emph{International conference on learning representations}, 2019.

\bibitem[Levine et~al.(2020)Levine, Kumar, Tucker, and Fu]{levine2020offline}
Sergey Levine, Aviral Kumar, George Tucker, and Justin Fu.
\newblock Offline reinforcement learning: Tutorial, review, and perspectives on open problems.
\newblock \emph{arXiv preprint arXiv:2005.01643}, 2020.

\bibitem[Margolis et~al.(2022)Margolis, Yang, Paigwar, Chen, and Agrawal]{margolis2022rapid}
Gabriel~B Margolis, Ge~Yang, Kartik Paigwar, Tao Chen, and Pulkit Agrawal.
\newblock Rapid locomotion via reinforcement learning.
\newblock \emph{arXiv preprint arXiv:2205.02824}, 2022.

\bibitem[Mendonca et~al.(2020)Mendonca, Geng, Finn, and Levine]{mendonca2020meta}
Russell Mendonca, Xinyang Geng, Chelsea Finn, and Sergey Levine.
\newblock Meta-reinforcement learning robust to distributional shift via model identification and experience relabeling.
\newblock \emph{arXiv preprint arXiv:2006.07178}, 2020.

\bibitem[Miki et~al.(2022)Miki, Lee, Hwangbo, Wellhausen, Koltun, and Hutter]{miki2022learning}
Takahiro Miki, Joonho Lee, Jemin Hwangbo, Lorenz Wellhausen, Vladlen Koltun, and Marco Hutter.
\newblock Learning robust perceptive locomotion for quadrupedal robots in the wild.
\newblock \emph{Science Robotics}, 2022.

\bibitem[Nachum et~al.(2018{\natexlab{a}})Nachum, Gu, Lee, and Levine]{nachum2018near}
Ofir Nachum, Shixiang Gu, Honglak Lee, and Sergey Levine.
\newblock Near-optimal representation learning for hierarchical reinforcement learning.
\newblock \emph{arXiv preprint arXiv:1810.01257}, 2018{\natexlab{a}}.

\bibitem[Nachum et~al.(2018{\natexlab{b}})Nachum, Gu, Lee, and Levine]{nachum2018data}
Ofir Nachum, Shixiang~Shane Gu, Honglak Lee, and Sergey Levine.
\newblock Data-efficient hierarchical reinforcement learning.
\newblock \emph{Advances in neural information processing systems}, 31, 2018{\natexlab{b}}.

\bibitem[Nagabandi et~al.(2018)Nagabandi, Clavera, Liu, Fearing, Abbeel, Levine, and Finn]{nagabandi2018learning}
Anusha Nagabandi, Ignasi Clavera, Simin Liu, Ronald~S Fearing, Pieter Abbeel, Sergey Levine, and Chelsea Finn.
\newblock Learning to adapt in dynamic, real-world environments through meta-reinforcement learning.
\newblock \emph{arXiv preprint arXiv:1803.11347}, 2018.

\bibitem[Nasiriany et~al.(2022)Nasiriany, Liu, and Zhu]{nasiriany2022augmenting}
Soroush Nasiriany, Huihan Liu, and Yuke Zhu.
\newblock Augmenting reinforcement learning with behavior primitives for diverse manipulation tasks.
\newblock In \emph{2022 International Conference on Robotics and Automation (ICRA)}, pp.\  7477--7484. IEEE, 2022.

\bibitem[Park \& Levine(2023)Park and Levine]{park2023predictable}
Seohong Park and Sergey Levine.
\newblock Predictable mdp abstraction for unsupervised model-based rl.
\newblock \emph{arXiv preprint arXiv:2302.03921}, 2023.

\bibitem[Peng et~al.(2018)Peng, Andrychowicz, Zaremba, and Abbeel]{peng2018sim}
Xue~Bin Peng, Marcin Andrychowicz, Wojciech Zaremba, and Pieter Abbeel.
\newblock Sim-to-real transfer of robotic control with dynamics randomization.
\newblock In \emph{2018 IEEE international conference on robotics and automation (ICRA)}, pp.\  3803--3810. IEEE, 2018.

\bibitem[Peng et~al.(2019)Peng, Chang, Zhang, Abbeel, and Levine]{peng2019mcp}
Xue~Bin Peng, Michael Chang, Grace Zhang, Pieter Abbeel, and Sergey Levine.
\newblock Mcp: Learning composable hierarchical control with multiplicative compositional policies.
\newblock \emph{Advances in Neural Information Processing Systems}, 32, 2019.

\bibitem[Peng et~al.(2020)Peng, Coumans, Zhang, Lee, Tan, and Levine]{peng2020learning}
Xue~Bin Peng, Erwin Coumans, Tingnan Zhang, Tsang-Wei Lee, Jie Tan, and Sergey Levine.
\newblock Learning agile robotic locomotion skills by imitating animals.
\newblock \emph{arXiv preprint arXiv:2004.00784}, 2020.

\bibitem[Pertsch et~al.(2021)Pertsch, Lee, Wu, and Lim]{pertsch2021guided}
Karl Pertsch, Youngwoon Lee, Yue Wu, and Joseph~J Lim.
\newblock Guided reinforcement learning with learned skills.
\newblock \emph{arXiv preprint arXiv:2107.10253}, 2021.

\bibitem[Rajeswaran et~al.(2016)Rajeswaran, Ghotra, Ravindran, and Levine]{rajeswaran2016epopt}
Aravind Rajeswaran, Sarvjeet Ghotra, Balaraman Ravindran, and Sergey Levine.
\newblock Epopt: Learning robust neural network policies using model ensembles.
\newblock \emph{arXiv preprint arXiv:1610.01283}, 2016.

\bibitem[Rothfuss et~al.(2018)Rothfuss, Lee, Clavera, Asfour, and Abbeel]{rothfuss2018promp}
Jonas Rothfuss, Dennis Lee, Ignasi Clavera, Tamim Asfour, and Pieter Abbeel.
\newblock Promp: Proximal meta-policy search.
\newblock \emph{arXiv preprint arXiv:1810.06784}, 2018.

\bibitem[Rusu et~al.(2016)Rusu, Rabinowitz, Desjardins, Soyer, Kirkpatrick, Kavukcuoglu, Pascanu, and Hadsell]{rusu2016progressive}
Andrei~A Rusu, Neil~C Rabinowitz, Guillaume Desjardins, Hubert Soyer, James Kirkpatrick, Koray Kavukcuoglu, Razvan Pascanu, and Raia Hadsell.
\newblock Progressive neural networks.
\newblock \emph{arXiv preprint arXiv:1606.04671}, 2016.

\bibitem[Rusu et~al.(2018)Rusu, Rao, Sygnowski, Vinyals, Pascanu, Osindero, and Hadsell]{rusu2018meta}
Andrei~A Rusu, Dushyant Rao, Jakub Sygnowski, Oriol Vinyals, Razvan Pascanu, Simon Osindero, and Raia Hadsell.
\newblock Meta-learning with latent embedding optimization.
\newblock \emph{arXiv preprint arXiv:1807.05960}, 2018.

\bibitem[Sadeghi \& Levine(2016)Sadeghi and Levine]{sadeghi2016cad2rl}
Fereshteh Sadeghi and Sergey Levine.
\newblock Cad2rl: Real single-image flight without a single real image.
\newblock \emph{arXiv preprint arXiv:1611.04201}, 2016.

\bibitem[Sharma et~al.(2019)Sharma, Gu, Levine, Kumar, and Hausman]{sharma2019dynamics}
Archit Sharma, Shixiang Gu, Sergey Levine, Vikash Kumar, and Karol Hausman.
\newblock Dynamics-aware unsupervised discovery of skills.
\newblock \emph{arXiv preprint arXiv:1907.01657}, 2019.

\bibitem[Sharma et~al.(2020)Sharma, Liang, Zhao, LaGrassa, and Kroemer]{sharma2020learning}
Mohit Sharma, Jacky Liang, Jialiang Zhao, Alex LaGrassa, and Oliver Kroemer.
\newblock Learning to compose hierarchical object-centric controllers for robotic manipulation.
\newblock \emph{arXiv preprint arXiv:2011.04627}, 2020.

\bibitem[Smith et~al.(2022)Smith, Kostrikov, and Levine]{smith2022walk}
Laura Smith, Ilya Kostrikov, and Sergey Levine.
\newblock A walk in the park: Learning to walk in 20 minutes with model-free reinforcement learning.
\newblock \emph{arXiv preprint arXiv:2208.07860}, 2022.

\bibitem[Song et~al.(2020)Song, Yang, Choromanski, Caluwaerts, Gao, Finn, and Tan]{song2020rapidly}
Xingyou Song, Yuxiang Yang, Krzysztof Choromanski, Ken Caluwaerts, Wenbo Gao, Chelsea Finn, and Jie Tan.
\newblock Rapidly adaptable legged robots via evolutionary meta-learning.
\newblock In \emph{2020 IEEE/RSJ International Conference on Intelligent Robots and Systems (IROS)}, pp.\  3769--3776. IEEE, 2020.

\bibitem[Strudel et~al.(2020)Strudel, Pashevich, Kalevatykh, Laptev, Sivic, and Schmid]{strudel2020learning}
Robin Strudel, Alexander Pashevich, Igor Kalevatykh, Ivan Laptev, Josef Sivic, and Cordelia Schmid.
\newblock Learning to combine primitive skills: A step towards versatile robotic manipulation.
\newblock In \emph{2020 IEEE International Conference on Robotics and Automation (ICRA)}, pp.\  4637--4643. IEEE, 2020.

\bibitem[Tan et~al.(2018)Tan, Zhang, Coumans, Iscen, Bai, Hafner, Bohez, and Vanhoucke]{tan2018sim}
Jie Tan, Tingnan Zhang, Erwin Coumans, Atil Iscen, Yunfei Bai, Danijar Hafner, Steven Bohez, and Vincent Vanhoucke.
\newblock Sim-to-real: Learning agile locomotion for quadruped robots.
\newblock \emph{arXiv preprint arXiv:1804.10332}, 2018.

\bibitem[Tobin et~al.(2017)Tobin, Fong, Ray, Schneider, Zaremba, and Abbeel]{tobin2017domain}
Josh Tobin, Rachel Fong, Alex Ray, Jonas Schneider, Wojciech Zaremba, and Pieter Abbeel.
\newblock Domain randomization for transferring deep neural networks from simulation to the real world.
\newblock In \emph{2017 IEEE/RSJ international conference on intelligent robots and systems (IROS)}, pp.\  23--30. IEEE, 2017.

\bibitem[Todorov et~al.(2012)Todorov, Erez, and Tassa]{todorov2012mujoco}
Emanuel Todorov, Tom Erez, and Yuval Tassa.
\newblock Mujoco: A physics engine for model-based control.
\newblock In \emph{2012 IEEE/RSJ international conference on intelligent robots and systems}, pp.\  5026--5033. IEEE, 2012.

\bibitem[Tunyasuvunakool et~al.(2020)Tunyasuvunakool, Muldal, Doron, Liu, Bohez, Merel, Erez, Lillicrap, Heess, and Tassa]{tunyasuvunakool2020}
Saran Tunyasuvunakool, Alistair Muldal, Yotam Doron, Siqi Liu, Steven Bohez, Josh Merel, Tom Erez, Timothy Lillicrap, Nicolas Heess, and Yuval Tassa.
\newblock dm\_control: Software and tasks for continuous control.
\newblock \emph{Software Impacts}, 6:\penalty0 100022, 2020.

\bibitem[Wang et~al.(2016)Wang, Kurth-Nelson, Tirumala, Soyer, Leibo, Munos, Blundell, Kumaran, and Botvinick]{wang2016learning}
Jane~X Wang, Zeb Kurth-Nelson, Dhruva Tirumala, Hubert Soyer, Joel~Z Leibo, Remi Munos, Charles Blundell, Dharshan Kumaran, and Matt Botvinick.
\newblock Learning to reinforcement learn.
\newblock \emph{arXiv preprint arXiv:1611.05763}, 2016.

\bibitem[Xie \& Finn(2021)Xie and Finn]{xie2021lifelong}
Annie Xie and Chelsea Finn.
\newblock Lifelong robotic reinforcement learning by retaining experiences.
\newblock \emph{arXiv preprint arXiv:2109.09180}, 2021.

\bibitem[Xie et~al.(2020)Xie, Harrison, and Finn]{xie2020deep}
Annie Xie, James Harrison, and Chelsea Finn.
\newblock Deep reinforcement learning amidst lifelong non-stationarity.
\newblock \emph{arXiv preprint arXiv:2006.10701}, 2020.

\bibitem[Xie et~al.(2021)Xie, Da, Van~de Panne, Babich, and Garg]{xie2021dynamics}
Zhaoming Xie, Xingye Da, Michiel Van~de Panne, Buck Babich, and Animesh Garg.
\newblock Dynamics randomization revisited: A case study for quadrupedal locomotion.
\newblock In \emph{2021 IEEE International Conference on Robotics and Automation (ICRA)}, pp.\  4955--4961. IEEE, 2021.

\bibitem[Yang et~al.(2023)Yang, Yang, and Wang]{yang2023neural}
Ruihan Yang, Ge~Yang, and Xiaolong Wang.
\newblock Neural volumetric memory for visual locomotion control.
\newblock In \emph{Proceedings of the IEEE/CVF Conference on Computer Vision and Pattern Recognition}, 2023.

\bibitem[Yoneda et~al.(2021)Yoneda, Yang, Walter, and Stadie]{yoneda2021invariance}
Takuma Yoneda, Ge~Yang, Matthew~R Walter, and Bradly Stadie.
\newblock Invariance through inference.
\newblock \emph{arXiv preprint arXiv:2112.08526}, 2021.

\bibitem[Yu et~al.(2020{\natexlab{a}})Yu, Thomas, Yu, Ermon, Zou, Levine, Finn, and Ma]{yu2020mopo}
Tianhe Yu, Garrett Thomas, Lantao Yu, Stefano Ermon, James~Y Zou, Sergey Levine, Chelsea Finn, and Tengyu Ma.
\newblock Mopo: Model-based offline policy optimization.
\newblock \emph{Advances in Neural Information Processing Systems}, 33:\penalty0 14129--14142, 2020{\natexlab{a}}.

\bibitem[Yu et~al.(2021)Yu, Kumar, Rafailov, Rajeswaran, Levine, and Finn]{yu2021combo}
Tianhe Yu, Aviral Kumar, Rafael Rafailov, Aravind Rajeswaran, Sergey Levine, and Chelsea Finn.
\newblock Combo: Conservative offline model-based policy optimization.
\newblock \emph{Advances in neural information processing systems}, 34:\penalty0 28954--28967, 2021.

\bibitem[Yu et~al.(2017)Yu, Tan, Liu, and Turk]{yu2017preparing}
Wenhao Yu, Jie Tan, C~Karen Liu, and Greg Turk.
\newblock Preparing for the unknown: Learning a universal policy with online system identification.
\newblock \emph{RSS}, 2017.

\bibitem[Yu et~al.(2019)Yu, Kumar, Turk, and Liu]{yu2019sim}
Wenhao Yu, Visak~CV Kumar, Greg Turk, and C~Karen Liu.
\newblock Sim-to-real transfer for biped locomotion.
\newblock In \emph{2019 ieee/rsj international conference on intelligent robots and systems (iros)}, pp.\  3503--3510. IEEE, 2019.

\bibitem[Yu et~al.(2020{\natexlab{b}})Yu, Tan, Bai, Coumans, and Ha]{yu2020learning}
Wenhao Yu, Jie Tan, Yunfei Bai, Erwin Coumans, and Sehoon Ha.
\newblock Learning fast adaptation with meta strategy optimization.
\newblock \emph{IEEE Robotics and Automation Letters}, 2020{\natexlab{b}}.

\bibitem[Zhuang et~al.(2023)Zhuang, Fu, Wang, Atkeson, Schwertfeger, Finn, and Zhao]{zhuang2023robot}
Ziwen Zhuang, Zipeng Fu, Jianren Wang, Christopher~G Atkeson, S{\"o}ren Schwertfeger, Chelsea Finn, and Hang Zhao.
\newblock Robot parkour learning.
\newblock In \emph{7th Annual Conference on Robot Learning}, 2023.

\end{thebibliography}
